\pdfoutput=1
\documentclass[11pt]{article}

\usepackage[margin=1in]{geometry}
\usepackage[T1]{fontenc}
\usepackage[utf8]{inputenc}
\usepackage{lmodern}
\usepackage{microtype}
\usepackage{amsmath,amssymb,amsthm,mathtools}
\usepackage{booktabs}
\usepackage{array}
\usepackage{tabularx}
\usepackage{enumitem}
\usepackage{float}
\usepackage{xcolor}
\usepackage{xurl}
\usepackage[numbers,sort&compress]{natbib}
\usepackage{hyperref}
\usepackage[nameinlink,capitalise]{cleveref}

\hypersetup{
  colorlinks=true,
  linkcolor=blue!55!black,
  citecolor=green!35!black,
  urlcolor=blue!65!black,
  pdftitle={From Attention Masks to Inert Zero-Vector Tokens: OAttention and O-Closure for Token Dynamics},
  pdfauthor={Heyang Gong}
}

\newtheorem{definition}{Definition}
\newtheorem{proposition}{Proposition}
\newtheorem{theorem}{Theorem}
\newtheorem{remark}{Remark}
\newcommand{\R}{\mathbb{R}}
\newcommand{\presence}{\rho}
\newcommand{\OA}{\operatorname{OAttention}}

\title{From Attention Masks to Inert Zero-Vector Tokens:\\
OAttention and O-Closure for Token Dynamics}
\author{Heyang Gong}
\date{}

\begin{document}
\maketitle

\begin{abstract}
Attention masks are indispensable relation-level controls: they specify which
query--source pairs may interact. They do not, however, provide a
representation-carried token state that is intrinsically non-participating at
the attention boundary. We assign each token hidden carrier \(h_i\) a smooth
active-presence coefficient
\(p_i=\lVert h_i\rVert^2/(\tau+\lVert h_i\rVert^2)\). The same coefficient has
two roles: it gates information emitted by token \(i\), and it determines the
mass with which token \(i\) enters computations shared with other tokens.

OAttention is the support-coupled attention realization of this rule. It gates
the receiver output by \(p_i\) and weights source \(j\) by \(p_j\) in both the
attention numerator and partition, while retaining the standard score,
visibility relation, exponential competition, and value aggregation. This
minimal completion makes the zero-vector token a zero element and yields exact
null-receiver, null-source insertion, self-attention insertion, and
empty-support properties. The same token-level presence gives local
O-components (OFFN, ONorm, and OInject), presence-weighted OStandardize, the
O-Closure law \(M(H\oplus0)=M(H)\oplus0\), and an OTransformer by residual and
compositional closure.

The canonical operator is checked by contract tests and an independent GPU
evaluation. In a zero-fine-tuning retrofit of a cloned pretrained TabPFN v3
regressor, calibrated hidden-carrier OAttention and Full-O variants change mean
RMSE by $+0.088\%$ and $+0.177\%$, respectively, over 18 matched
dataset--seed cases. A two-block ablation further shows that OAttention alone
does not preserve a NULL state through ordinary host components, whereas the
declared OTransformer path does. These are scoped tests of exactness, active-path
compatibility, and compositional necessity; they do not establish universal
no-loss, arbitrary-host closure, learned attraction to the origin, or a general
semantics for missing values.
\end{abstract}

\section{Introduction}
\label{sec:introduction}

Suppressing information is a basic requirement in attention-based models. The
standard mechanism is a mask \(m_{ij}\), supplied to the operator to specify
which query--source relations are allowed. Such masks are indispensable for
causal structure, routing, permissions, and target visibility, but they do not
provide a token state whose non-participation is carried by its representation.

Padding, absent support, unused routing slots, and deliberately empty latent
states motivate a complementary requirement. At a declared operator boundary,
an empty state should emit no output as a query and contribute no support mass
along any relation for which it would otherwise be visible. A mask therefore
answers a relation-level question---which token pairs may interact---whereas a
structural NULL answers a state-level question---whether this representation
participates at all. The two controls are complementary: state-level nullness
does not replace causal, routing, permission, or target-visibility masks.

Standard scaled dot-product attention does not satisfy this requirement
\citep{vaswani2017attention}. For visible sources,
\begin{equation}
  \alpha_{ij}
  =\frac{\exp(q_i^\top k_j/\sqrt d)}
  {\sum_t\exp(q_i^\top k_t/\sqrt d)},
  \qquad
  y_i=\sum_j\alpha_{ij}v_j.
  \label{eq:sdpa}
\end{equation}
If \(q_i=0\), every finite logit is zero and the output is the average of the
visible values, not zero. If a source \((k_0,v_0)=(0,0)\) is appended, it adds
\(\exp(0)=1\) to the denominator. Its direct value contribution is zero, but
all pre-existing weights change. The origin is therefore an ordinary state of
softmax attention rather than an inert one.

We give state-level non-participation a single token-level coordinate. At a
declared hidden-carrier boundary, token \(i\) has presence
\begin{equation}
  p_i=\presence_\tau(h_i)
  =\frac{\lVert h_i\rVert_2^2}{\tau+\lVert h_i\rVert_2^2}.
  \label{eq:intro-token-presence}
\end{equation}
The coefficient has two roles. \emph{Local update participation} determines
whether token \(i\) emits newly generated information.
\emph{Context support participation} determines whether that token contributes
mass to a shared support, partition, or statistic. The same \(p_i\), rather
than a new gate for each component, governs both roles.

We introduce \emph{OAttention}, where O denotes the origin of the
representation space, as the support-coupled attention instance. Query and key
projections still determine the score geometry, but the receiver hidden state
supplies the output factor \(p_i\) and each source hidden state supplies its
support mass \(p_j\). OAttention otherwise retains the standard score,
externally supplied visibility relation, exponential competition, and value
aggregation. The visibility mask says whether an edge is permitted; token
presence says whether an otherwise visible state participates on that edge.

The broader principle is necessary because attention alone cannot make a
model null consistent. An additive position encoding can map zero to a
nonzero vector; an affine normalization or biased feed-forward branch can
manufacture a nonzero update; and an ordinary cross-token statistic can allow
an inserted zero to change every pre-existing output. O-Closure distinguishes
token-local maps, which must absorb zero, from support-coupled maps, which must
also assign zero mass to null sources. OFFN, ONorm, OInject, and
OStandardize instantiate these two cases, and closure under residual addition
and composition gives a conditional construction of an OTransformer.

Our contributions are:
\begin{enumerate}[leftmargin=1.6em,itemsep=0.3em]
  \item We formulate null consistency as exact operator properties that
        separate representation-carried, state-level participation from
        externally supplied, relation-level visibility masks; from one
        hidden-carrier presence coefficient we derive OAttention and prove its
        null and active-limit properties.
  \item We identify local update participation and context support
        participation as two placements of the same token presence, formulate
        their zero-extension law, and derive OFFN, ONorm, OInject,
        OStandardize, and a conditional OTransformer by residual and
        compositional closure.
  \item We test three distinct claims: exact implementation of the canonical
        operator, calibrated near-identity compatibility in a frozen pretrained
        TabPFN v3 host, and the necessity of whole-path closure in a scoped
        OTransformer ablation.
\end{enumerate}

These claims deliberately have different scopes. Theorems establish the
declared algebraic contracts; finite-precision evaluations check their
implementation; and task metrics are bounded compatibility observations rather
than universal non-inferiority tests. We do not claim closure for arbitrary
tokenizers, routers, caches, or
readouts; learned zero-attractor dynamics; or a general semantics for missing
values.

\section{Related work}
\label{sec:related-work}

\paragraph{Sparse and no-op attention.}
Sparsemax and entmax can assign individual sources zero weight, but their
simplex constraint prevents an all-zero row
\citep{martins2016sparsemax,peters2019entmax}. Sigmoid attention removes
row-wise normalization, but a zero dot product is not intrinsically null
\citep{ramapuram2025sigmoid}. ReLA and Softpick instead allow zero aggregate
updates through rectified scores
\citep{zhang2021rela,zuhri2026softpick}. These mechanisms show that attention
need not always emit a convex combination. Our question is narrower: whether a
representation-carried state can be inert as both receiver and source while
the active regime retains vanilla softmax. The operator-level distinction from
score-space no-op attention is made after OAttention is defined.

\paragraph{Gates, routing, sinks, and spare capacity.}
Output and value-state gates attenuate attention updates
\citep{bondarenko2023quantizable,qiu2025gated,bu2026vga}, while conditional
computation routes tokens through selected feed-forward or attention branches
\citep{ainslie2023colt5,zeng2023skip}. Attention-sink studies analyze
content-light destinations and default no-op behavior
\citep{clark2019bert,chen2021maskalign,xiao2023streaming,gu2024sink,
ranmilo2026sinks,fesser2026sinks}. Softmax-plus-one supplies fixed spare
partition mass, and StableMask modifies causal masking with pseudo-attention
values \citep{miller2023offbyone,hu2024outeffhop,yin2024stablemask}. These lines
of work control branches, destinations, or relations; they do not by themselves
identify an exact-zero carrier with the null-insertion contract studied here.
A trainable-key ``zero token'' used for cyclic refinement is likewise an
information-bearing control state rather than a structural NULL
\citep{li2025zerotoken}.

\paragraph{Norms, sets, and missing data.}
Vector norms have been used to analyze attention contributions and normalize
query--key geometry
\citep{kobayashi2020norms,henry2020qknorm,govindarajan2026quest}; capsule
squash provides a related radial factor \citep{sabour2017capsules}. Deep Sets
and Set Transformer provide the broader context of permutation-aware set
processing \citep{zaheer2017deepsets,lee2019settransformer}, although
permutation symmetry alone does not imply invariance to adjoining an unmasked
zero element. Missing-data models may either suppress missing entries or encode
missingness as information \citep{caruso2026naim,chang2025mfan}; this motivates
our strict separation between structural NULL and informative missingness.

\section{From relation masks to a null state}
\label{sec:null-consistency}

Let an attention operator consume hidden receiver carriers
\((h_i)_{i\in I}\), hidden source carriers \((h_j)_{j\in J}\), their projected
queries, keys, and values, and a visibility relation
\(m_{ij}\in\{0,1\}\). The mask is indexed by query--source relations and is
supplied as a separate operator input rather than implied by token state. A
\emph{zero-vector token} at the canonical boundary is \(h=0\). A projected
zero in one head, a zero value with a nonzero carrier, an observed scalar
equal to zero, and a learned \texttt{[MASK]} token are different objects.

\begin{definition}[Null consistency]
An attention operator is null consistent at its declared carrier boundary if
it satisfies:
\begin{enumerate}[leftmargin=1.5em,itemsep=0.2em]
  \item \textbf{Null receiver:} \(h_i=0\) implies \(y_i=0\).
  \item \textbf{Null-source insertion:} inserting any finite number of source
        carriers \(h_j=0\) changes no pre-existing output or old source weight.
  \item \textbf{Self-attention insertion:} inserting a zero hidden carrier
        leaves old outputs unchanged and gives the new receiver a zero output.
  \item \textbf{Empty support:} an empty or fully masked source set returns a
        finite zero output and zero weights.
\end{enumerate}
\end{definition}

Null consistency is an operator effect, not a replacement for masking. For all
pre-existing receivers, inserting \(h_j=0\) should be equivalent to removing it
or masking its entire source column, while leaving old--old relations and
metadata unchanged. The state and relation remain semantically distinct: a
learned \texttt{[MASK]} embedding may deliberately carry information, and
masking a source column does not deactivate the receiver row at the same
position. A full self-attention NULL also has zero receiver presence, so its
emitted output is zero regardless of visible context; this does not imply that
its internal attention weights equal those of a fully masked row. A query
placeholder that must read context can instead be receiver-active and
source-inactive; it is not a full zero-vector token.

Standard attention violates the first two properties. For a visible zero
carrier under zero-preserving projections, the old denominator \(Z_i\) becomes
\(Z_i+1\), and every old weight is multiplied by \(Z_i/(Z_i+1)\). For a zero
receiver carrier, the projected query is zero and the score function loses all
directional distinctions but not the aggregation itself. These failures arise
from the normalization topology, not from numerical instability.

\section{Token presence and OAttention}
\label{sec:oattention}

\subsection{Smooth radial presence}

For \(\tau>0\), define
\begin{equation}
  \presence_\tau(x)
  =\frac{\lVert x\rVert_2^2}{\tau+\lVert x\rVert_2^2}.
  \label{eq:presence}
\end{equation}
The map is radial, differentiable, bounded in \([0,1)\), exactly zero only at
the origin, and approaches one on any fixed nonzero vector as
\(\tau\to0^+\). Near the origin,
\(\presence_\tau(x)=\lVert x\rVert^2/\tau+O(\lVert x\rVert^4)\), so
attenuation is continuous and quadratic rather than thresholded.

The norm is evaluated before any normalization that would erase magnitude.
The construction adds no learned projection, but \(\tau\) remains a scale
hyperparameter and the learned representations determine the gate value.

\paragraph{Relation to norm-derived mechanisms.}
Unlike query--key normalization, which changes score geometry, OAttention reads
the carrier norm before normalization and assigns the resulting scalar a
participation semantics. The radial factor in Eq.~\eqref{eq:presence} also
appears in capsule squash, but capsule squash uses it to rescale a vector rather
than to couple receiver output with source support.

\subsection{Canonical hidden-carrier operator}

Let \(h_i\) be a receiver carrier and \(h_j\) a source carrier. Suppressing
head indices, define
\begin{equation}
  p_i=\presence_\tau(h_i),\qquad
  q_i=W_Qh_i,\quad k_j=W_Kh_j,\quad v_j=W_Vh_j,\qquad
  s_{ij}=\frac{q_i^\top k_j}{\sqrt d}.
  \label{eq:hidden-carrier-projections}
\end{equation}
For an externally supplied visibility relation \(m_{ij}\) and a small
\(\varepsilon_{\mathrm{den}}>0\), OAttention is
\begin{align}
  u_{ij}&=m_{ij}p_j\exp(s_{ij}),\\
  w_{ij}&=\frac{u_{ij}}
    {\varepsilon_{\mathrm{den}}+\sum_tu_{it}},\\
  \OA_i^H(H;m)
    &=p_i\sum_jw_{ij}v_j
     =p_i
       \frac{\sum_jm_{ij}p_j\exp(s_{ij})v_j}
       {\varepsilon_{\mathrm{den}}+
        \sum_jm_{ij}p_j\exp(s_{ij})}.
  \label{eq:oattention}
\end{align}
One token coefficient now governs both roles. The receiver factor \(p_i\)
controls emitted output; the source factor \(p_j\) enters both numerator and
partition and therefore controls support mass. It is shared across heads at
this hidden boundary. The denominator stabilizer only totalizes empty support;
it is not a learned NULL destination or an attention objective.

\begin{proposition}[Hidden-carrier null receiver]
For all finite sources and masks, \(h_i=0\) implies
\(\OA_i^H(H;m)=0\).
\end{proposition}
\begin{proof}
Equation~\eqref{eq:presence} gives \(p_i=\presence_\tau(0)=0\), which
multiplies the complete aggregate in Eq.~\eqref{eq:oattention}.
\end{proof}

\begin{proposition}[Hidden-carrier null-source insertion]
Appending any finite number of source carriers \(h_j=0\), with arbitrary
visibility bits and compatible old metadata, leaves every pre-existing output
and old source weight unchanged.
\end{proposition}
\begin{proof}
Each inserted source has \(p_j=0\), hence \(u_{ij}=0\) for every receiver. It
contributes neither to the numerator nor to the partition; all old terms and
receiver factors are unchanged.
\end{proof}

\begin{proposition}[Self-attention insertion and empty support]
Inserting a hidden zero carrier in self-attention leaves all old outputs
unchanged and gives the inserted receiver output zero. If effective support is
empty, all weights and outputs are finite zero.
\end{proposition}
\begin{proof}
The old-output statement follows from null-source insertion, and the new
output follows from the null-receiver proposition. With empty support, every
\(u_{ij}=0\); the positive stabilizer makes each weight zero.
\end{proof}

\begin{remark}[Projection-boundary variation]
Presence can instead be evaluated after Q/K projection, separately by role and
head. That construction satisfies an analogous contract at a different
boundary but is not a notation variant of Eq.~\eqref{eq:oattention}: a nonzero
hidden carrier may project to zero in one head. We keep this variation, its
formula, and its independent evidence in
\Cref{app:qkv-reference-form,app:qkv-reference-results}; the canonical argument
below always refers to token-level hidden-carrier presence.
\end{remark}

\paragraph{Comparison with score-space no-op attention.}
At a bias-free dot-product boundary, ReLA makes a zero query or key inert in its
rectified-score core and can produce an all-zero aggregate; Softpick also
permits zero aggregate updates. Their source selection is determined by the
sign of each query--key score and changes the score-normalization regime. In
OAttention, source presence is carried by \(h_j\), is independent of which
receiver reads that source, and modulates an otherwise unchanged exponential
competition. We therefore treat ReLA as the closest score-space no-op
precedent, not as the same operator or as evidence of poor general performance.

\subsection{Vanilla-consistent active-state limit}

For a fixed finite collection of nonzero hidden carriers,
\(p_i,p_j\to1\) as \(\tau\to0^+\). If the ordinary softmax partition is
positive, taking \(\varepsilon_{\mathrm{den}}\to0^+\) yields
\begin{equation}
  \OA_i^H(H;m)\longrightarrow
  \sum_j
  \frac{m_{ij}\exp(s_{ij})}{\sum_tm_{it}\exp(s_{it})}v_j.
  \label{eq:vanilla-limit}
\end{equation}
At finite \(\tau\), active states can still be attenuated; the limit does not
imply exact equality at a chosen finite scale.

\paragraph{Comparison with spare-mass normalizers.}
Softmax-plus-one adds a fixed non-data term to the partition and therefore
shrinks an ordinary active row by a factor of the form \(Z/(Z+\lambda)\).
StableMask changes the causal visibility construction and introduces
pseudo-attention values. By contrast, \(\varepsilon_{\mathrm{den}}\) in
Eq.~\eqref{eq:oattention} only totalizes empty support and vanishes in the
active-state limit; a zero source is removed through token presence rather than
represented by a competing sink.

\section{From OAttention to O-closed token dynamics}
\label{sec:o-closure}

Within OAttention, relation-level visibility and state-level presence play
different roles. The mask selects permitted query--source edges; receiver
presence gates the emitted output, while source presence enters the normalized
support along otherwise visible edges. These are not attention-specific gates:
they are two placements of one operator-independent notion of token
participation. A token participates both by changing its own state and by
contributing mass to computations shared with other tokens.

Attention alone cannot establish this semantics at the model level. A biased
FFN or additive encoding can map zero to a nonzero state, while a cross-token
statistic can allow an inserted zero to alter existing outputs. A model-level
zero element therefore requires both zero absorption in token-local maps and
zero support mass in token-coupled maps.

\subsection{One token presence, two operator roles}

At the current hidden-carrier boundary, define once
\begin{equation}
  p_i=\presence_\tau(h_i).
  \label{eq:role-presence}
\end{equation}
Every O-component consuming this boundary state uses that coefficient. A
token-local branch uses \(p_i\) to control emitted update; a coupled component
uses the source coefficients \(p_j\) to define its support measure. Presence
is recomputed when token dynamics produce a new hidden carrier at the next
declared boundary, but it is not recomputed separately from Q/K projections
inside canonical OAttention.

\subsection{Zero extension}

Let \(H_I=(h_i)_{i\in I}\) be a finite indexed token family. For a new index
\(j\notin I\), let \(\iota_jH_I=H_I\oplus_j0\) denote extension by an exact
zero token. Metadata are called \emph{compatible} when extension leaves the
metadata of all old indices unchanged; in particular, an insertion operation
that renumbers absolute positions is not compatible under this definition.

\begin{definition}[O-Closure]
A family of token operators \(M=(M_I)_I\) is O-closed if, for every finite
\(I\), every \(j\notin I\), and compatible metadata \(c,c^+\),
\begin{equation}
  M_{I\cup\{j\}}(\iota_jH_I;c^+)
  =\iota_j M_I(H_I;c).
  \label{eq:o-closure}
\end{equation}
\end{definition}

The inserted coordinate on the right-hand side is zero, expressing
\emph{state-zero absorption}. Equality on the old coordinates expresses
\emph{contextual erasure}: adding a zero source has the same operator effect as
removing it from the relevant support. The two statements are complementary.
The zero state is absorbing at its own position and neutral when viewed as
context by other positions.

For deterministic modules, Eq.~\eqref{eq:o-closure} is the compact
zero-extension law. A stochastic version can be stated for the conditional
kernel: the inserted state has conditional law \(\delta_0\), and the joint law
of the old states is unchanged. This explains the token-dynamics reading
without conflating a full NULL state with a query placeholder that must be
predicted from context.

\subsection{Local and support-coupled lifts}

For a token-local map, a gated branch has the form
\begin{equation}
  \mathcal{G}_O[F](h_i)
  =p_iF(h_i).
  \label{eq:branch-lift}
\end{equation}
If \(F(0)\) is finite, this branch emits zero at a null input even when \(F\)
contains a bias. A residual lift is
\begin{equation}
  \mathcal{R}_O[F](h_i)
  =h_i+\presence_\tau(h_i)F(h_i).
  \label{eq:residual-lift}
\end{equation}

Output gating alone is insufficient for a token-coupled operator: a zero
source may still alter a shared denominator or statistic. Source presence must
enter the support measure. Retaining token identity and state as marks, define
\begin{equation}
  \nu_H^O=\sum_{i\in I}p_i\,\delta_{(i,h_i)}.
  \label{eq:o-measure}
\end{equation}
OAttention is a query-dependent exponential tilt of this same token measure,
\begin{equation}
  \OA_i^H
  =p_i
  \frac{\sum_jm_{ij}p_j e^{s_{ij}}v_j}
       {\varepsilon_{\mathrm{den}}+\sum_jm_{ij}p_j e^{s_{ij}}}.
  \label{eq:oattention-measure}
\end{equation}
The receiver use and source-support use are two placements of the same
token-level presence, not separately inferred attention gates.

\subsection{Residual and compositional closure}

\begin{theorem}[Closure]
Let \(M\) and \(N\) satisfy Eq.~\eqref{eq:o-closure} on compatible token
families. Then \(N\circ M\) is O-closed. If an update field \(U\) is O-closed,
then the residual map \(I+U\) is O-closed.
\end{theorem}
\begin{proof}
For composition,
\[
  (N\circ M)(H\oplus_j0)
  =N(M(H)\oplus_j0)
  =N(M(H))\oplus_j0.
\]
For the residual map, the inserted coordinate is \(0+0=0\), and both the
identity and update fields agree on every old coordinate before and after zero
extension.
\end{proof}

The theorem is conditional on complete module closure. If an update is
context-coupled, multiplying only its output at position \(i\) can establish
state absorption but not contextual erasure; its aggregation must also be
compatible with the presence-weighted support in Eq.~\eqref{eq:o-measure}.

\subsection{O-component realizations}
\label{sec:o-components}

\subsubsection{OFFN, ONorm, and OInject}

Let \(h_i\) be the state entering the declared component boundary and let
\(p_i=\presence_\tau(h_i)\).

For a feed-forward branch, including a branch with internal biases, define
\begin{equation}
  \operatorname{OFFN}_{\mathrm{upd}}(h_i)
  =p_i\operatorname{FFN}(h_i),
  \qquad
  h_i^+=h_i+\operatorname{OFFN}_{\mathrm{upd}}(h_i).
  \label{eq:offn}
\end{equation}
The notation distinguishes the gated update from the complete residual
transformation.

For token-local normalization \(N\), such as LayerNorm over channels, define
\begin{equation}
  \operatorname{ONorm}(h_i)=p_iN(h_i).
  \label{eq:onorm}
\end{equation}
This is a gated replacement or branch map, not a cross-token statistic.

For additive position, schema, or label metadata \(e_i\), we use
\begin{equation}
  \operatorname{OInject}(h_i,e_i)=h_i+p_ie_i.
  \label{eq:oinject}
\end{equation}
Only the newly injected branch is gated; the carrier is retained. The
alternative \(p_i(h_i+e_i)\) contracts the carrier itself and may be useful in
an attractor-oriented design, but it is not used in the evaluated
OTransformer.

Presence is recomputed from the current state at every declared boundary.
Consequently, an exact zero produced by one component is absorbed by the next.
The same convention can deactivate an active representation that is mapped
exactly to the origin, which motivates the collapse diagnostics and
limitations in \Cref{sec:limitations}.

\subsubsection{OStandardize}

Standardization across a token or cell axis is support-coupled. For
\(p_i=\presence_\tau(h_i)\), \(S=\sum_ip_i\), and \(S>0\), define
\begin{align}
  \mu_O
    &=\frac{\sum_ip_ih_i}{S}, \\
  v_O
    &=\frac{\sum_ip_i(h_i-\mu_O)^{\odot2}}{S}, \\
  z_i^O
    &=p_i\left[
      \gamma\odot\frac{h_i-\mu_O}
      {\sqrt{v_O+\varepsilon_{\mathrm{var}}}}+\beta
    \right].
  \label{eq:ostandardize}
\end{align}
The definition is totalized at empty effective support by
\begin{equation}
  S=0
  \quad\Longrightarrow\quad
  \mu_O=0,\qquad v_O=0,\qquad z_i^O=0\ \text{for all }i.
  \label{eq:ostandardize-empty}
\end{equation}
The implementation realizes this branch with a safe positive divisor followed
by explicit zero selection. An inserted zero token has \(p=0\), so it changes neither \(S\),
\(\mu_O\), nor \(v_O\), leaves all old outputs unchanged, and receives a zero
output. The variance uses the presence-weighted mean of squared residuals;
applying the radial map to \(\mu\) or \(\sigma\) itself would not establish
this invariance.

OStandardize and ONorm have different topologies. ONorm is token-local and
gates its output. OStandardize couples positions and must alter both the
statistical numerator and its normalizer. The final multiplication by \(p_i\)
means that we claim insertion-invariant weighted moments and zero null output,
not exact unit variance of the final gated activations under the same measure.

\subsubsection{OTransformer by compositional closure}

An OTransformer is a token architecture in which every declared token-local
path is zero absorbing and every declared support-coupled path assigns zero
mass to zero states. Its canonical attention component is hidden-carrier
OAttention, so the same \(p_i\) used by local components also controls
attention participation. The evaluated two-block instantiation uses
canonical hidden-carrier OAttention together with OInject, ONorm, and OFFN.
It tests compositional zero closure for that declared host path; task metrics
remain descriptive rather than a universal no-loss claim.
OStandardize is evaluated independently and is not part of that host.

Under compatible tokenization and metadata, applying the closure theorem
through \(L\) blocks gives
\begin{equation}
  \operatorname{OTransformer}_L(H\oplus_j0)
  =\operatorname{OTransformer}_L(H)\oplus_j0.
  \label{eq:otransformer-closure}
\end{equation}
For another architecture, this conclusion additionally requires its routing,
pooling, compression, caching, and token-level readout paths to satisfy the
same law. It does not imply that a biased task head vanishes on an all-null
input or that every missing value should be represented by zero.

\subsection{Validity domain and host-model conditions}
\label{sec:implementation}

The canonical contract begins at the declared hidden-carrier boundary: one
presence value is shared across heads for each token. The implementation
explicitly totalizes zero-length and fully masked support. Precision,
mask-layout, grouped-query broadcasting, and API details are reported in the
supplement.

Insertion invariance is exact in real arithmetic. Dense floating-point
reduction order can change when columns are inserted, so equality of old
outputs is evaluated with dtype-appropriate tolerances. Source weights and
receiver outputs at the declared exact-zero boundary remain exactly zero. The
deterministic operator is defined before dropout; a
training kernel that changes random-number indexing after insertion must state
whether it promises pathwise coupling or only distributional equivalence.

The current standalone wrapper uses bias-free Q/K/V and output projections.
Canonical source and receiver gating does not require a hidden zero to remain
zero inside Q/K/V projection, but its returned update still requires a
zero-preserving output projection or a final carrier gate. This is sufficient
for the attention primitive, not for an arbitrary host model.
Additive metadata, affine normalization, biased branches, routing, pooling,
and readout can all violate Eq.~\eqref{eq:o-closure}. The O-component wrappers
preserve biased local modules by gating the newly generated branch, but every
support-coupled path still requires its own presence-aware construction.

The structural NULL also remains distinct from data semantics. An observed
numeric zero may be informative and should normally receive an active carrier.
A learned \texttt{[MASK]} embedding is an information-bearing role token. A
missingness indicator can itself be predictive. Mapping NaN or missing values
to the origin is appropriate only when the application declares their carrier
to be semantically inert; OAttention does not make that declaration.

\section{Evaluation}
\label{sec:evaluation}

We align each experiment with one of three claims:
\begin{enumerate}[leftmargin=1.6em,itemsep=0.2em]
  \item \textbf{Exactness:} does the canonical implementation realize the
        declared null and insertion contracts at finite precision?
  \item \textbf{Active-path compatibility:} can token presence be introduced
        into a frozen pretrained host as a calibrated near-identity change?
  \item \textbf{Compositional necessity:} does attention alone fail when
        ordinary host components reactivate NULL, and does the declared
        O-closed path repair that failure?
\end{enumerate}
No superiority objective or post-hoc equivalence margin is imposed. Task
metrics are descriptive matched comparisons; the primary algebraic outcomes
are null-state preservation and inert insertion.

\subsection{Exact operator contracts}

An independent DGX2 operator sweep evaluates the implementation of
Eq.~\eqref{eq:oattention}. It covers fp32, bf16, and fp16; ordinary multi-head
and grouped-query layouts; self- and cross-attention; empty and all-null
support; finite gradients; and insertion at multiple positions.

\begin{table}[H]
  \centering
  \small
  \caption{Independent hidden-carrier OAttention operator evaluation. Errors are
  maxima on an NVIDIA GB10 GPU.}
  \label{tab:hidden-carrier-results}
  \begin{tabularx}{\linewidth}{@{}lX@{}}
    \toprule
    Property & Result \\
    \midrule
    Canonical equation & Output and weight \(L_\infty\) error
      \(\le 8.94\times10^{-8}\) against direct evaluation \\
    Null insertion & Old-output \(L_\infty\le4.47\times10^{-8}\);
      old-weight \(L_\infty\le5.96\times10^{-8}\) \\
    Exact null boundary & Inserted receiver output and source weight: \(0\) \\
    Empty/all-null support & Output and weights finite exact zero \\
    Interfaces & 6/6 dtype--head cases finite; MHA, GQA, fp32, bf16, fp16,
      cross-attention, and tested gradients \\
    \bottomrule
  \end{tabularx}
\end{table}

The propositions establish the algebraic contract; this evaluation checks that
the implementation realizes it at finite precision.

The standalone OStandardize sweep checks the other support-coupled topology:
presence-weighted moments and old outputs are invariant to tested zero
insertions, inserted outputs are exact zero, and all tested support and dtype
cases are finite. Full tensor shapes, permutation checks, singleton centering,
and the selected next-boundary semantics are reported in the supplement. This
operator is not folded into the task host.

\subsection{Pretrained active-path compatibility: TabPFN v3}
\label{sec:tabpfn-v3-results}

To test whether token presence can be introduced without retraining or
redesigning the active computation, we cloned the pretrained TabPFN v3
regressor and left its learned weights and preprocessing unchanged. The
baseline, hidden-carrier OAttention, and Full-O instances each use one copy of
the same checkpoint. Full-O adds hidden-carrier OAttention,
ONorm, OFFN, and hidden-carrier gating of the two target-encoding additions in
the v3 forward path; OStandardize is deliberately not inserted into the
preprocessing pipeline. We evaluate six cached OpenML regression datasets
(IDs 560, 44959, 505, 507, 227, and 189), three seeds \(11,23,37\), one
estimator, and deterministic caps of 512 training and 256 test rows. No arm is
fine-tuned.

The main comparison uses \(\tau=10^{-8}\), a near-identity active-path
calibration that nevertheless maps an exact zero carrier to exactly zero. Across
the 18 dataset--seed pairs, hidden-carrier OAttention changes mean RMSE by
$+0.088\%$ and mean $R^2$ by $+8.4\times10^{-5}$; Full-O changes them by
$+0.177\%$ and $+1.2\times10^{-5}$, respectively. These are bounded
inference observations, not a universal non-inferiority test. This is the
principal model-level evidence for minimal perturbation; the adapter and
OTransformer studies serve as complementary boundary and composition checks.

\begin{table}[t]
  \centering
  \scriptsize
  \caption{Near-identity hidden-carrier retrofit of pretrained TabPFN v3
  (\(\tau=10^{-8}\)). Values are means over three seeds. RMSE is in the
  original target units; $R^2$ is computed on each test split.}
  \label{tab:tabpfn-v3-near-identity}
  \begin{tabular}{@{}lrrrrrr@{}}
    \toprule
    Dataset & Base RMSE & OA RMSE & Full-O RMSE & Base $R^2$ & OA $R^2$ & Full-O $R^2$ \\
    \midrule
    bodyfat & 1.231549 & 1.230964 & 1.232571 & .961179 & .961228 & .961112 \\
    concrete & 4.127869 & 4.133669 & 4.129971 & .936922 & .936758 & .936862 \\
    cpu\_small & 3.023569 & 3.022008 & 3.023061 & .971290 & .971319 & .971299 \\
    kin8nm & .091274 & .091292 & .091278 & .877050 & .877005 & .877038 \\
    space\_ga & .093484 & .093366 & .093445 & .740970 & .741612 & .741182 \\
    tecator & .422813 & .424202 & .424954 & .999087 & .999082 & .999078 \\
    \bottomrule
  \end{tabular}
\end{table}

The corresponding finite-gate stress test is reported in
\Cref{app:tabpfn-v3-stress}. At \(\tau=1\), OAttention alone remains close to
the baseline, whereas repeatedly gating every local branch produces sizeable
active-path changes in Full-O. This separation is useful evidence about scale
calibration, not evidence against the zero-extension contract.

\paragraph{Complementary adapter and insertion checks.}

These checks calibrate the extension without carrying the main empirical
argument. A seven-task, 21-pair adapter rerun shows mixed, small active-path
metric changes (\Cref{app:adapter-table}). In 15 task--seed test-time insertion
pairs, hidden-carrier OAttention preserved every appended NULL state exactly,
whereas standard attention activated the new positions and produced larger
prediction shifts (\Cref{app:insertion-check}). A separate 120-fit
training-time zero-column study gives the same bounded contrast
(\Cref{app:zero-column-table}). These results are supporting diagnostics, not
a claim of universal task-level non-inferiority.

\subsection{Why attention alone is insufficient: OTransformer closure}

The OTransformer study uses canonical hidden-carrier OAttention in a two-block,
32-dimensional, four-head feature-token host on the same five datasets and
seeds, with \(k\in\{0,1,4\}\) inserted zeros during training. The Standard,
OA-only, and OTransformer arms share exact initialization. A separately
initialized zero-preserving structural control is included only as an
architectural reference. The OA-only arm isolates the failure of
attention-level closure to repair ordinary additive, affine, and biased paths.

\begin{table}[H]
  \centering
  \small
  \caption{Scoped OTransformer study. Maxima are over the evaluated
  task--seed--zero-count matrix for \(k>0\). The structural control is not part
  of the initialization-matched three-arm comparison.}
  \label{tab:otransformer-results}
  \begin{tabularx}{\linewidth}{@{}lXXX@{}}
    \toprule
    Arm & Host components & Max inserted-state \(L_\infty\) & Interpretation \\
    \midrule
    Standard & Softmax, additive encoding, affine norm, biased FFN & 7.982 & Ordinary host \\
    OA-only & Hidden-carrier OA; other components ordinary & 7.983 & Attention alone is insufficient \\
    OTransformer & Hidden-carrier OA, OInject, ONorm, OFFN & 0 & Declared O-closed path \\
    Structural control & Hidden-carrier OA, no additive encoding, affine-free norm, bias-free FFN & 0 & Separately initialized comparator \\
    \bottomrule
  \end{tabularx}
\end{table}

All 180 fits were finite. In the OTransformer arm, the maximum old-token state
\(L_\infty\) shift was \(2.47\times10^{-3}\) and the maximum prediction
\(L_\infty\) shift
was \(2.08\times10^{-3}\). These are finite-precision diagnostics in one host,
not exact empirical equality or a universal task-level guarantee.

\section{Limitations and open questions}
\label{sec:limitations}

\paragraph{Presence is an operational convention.}
A learned norm is not guaranteed to represent epistemic absence. Useful weak
signals may have low magnitude, and optimization may avoid the origin. The
radial map supplies a differentiable route to lower participation, but it does
not prove that training discovers or uses a null state. Any zero-attractor
claim requires longitudinal measurements of norms, presence, and transition
dynamics under an explicit learning objective.

\paragraph{Finite-scale attenuation.}
At finite \(\tau\), a gate can attenuate weak nonzero states, and repeated
gates may compound this effect. This is a deliberate operational trade-off:
the exact origin remains a mathematical NULL state, while ordinary active
states should stay close to the standard path when \(\tau\) is below their
characteristic norm scale. The intended regime is therefore not to suppress
weak information indiscriminately, but to introduce an exact null token while
preserving active computation as closely as possible. Wider studies should
still examine task-dependent calibration and weak-feature survival.

\paragraph{Exact-zero collisions.}
Cross-token centering can map an active token exactly to zero---for example, a
singleton support or a token equal to the weighted mean. Under the selected
layerwise semantics, the next component treats that representation as NULL.
This is a genuine consequence rather than a numerical corner case. A system
that must preserve semantic activity across centering should carry an explicit
pre-standardization presence side channel instead of re-inferring it from the
centered vector.

\paragraph{Host scope and metadata.}
The whole-model theorem is conditional on compatible metadata and closure of
every path. Absolute-position renumbering, lossy token merging, caches,
sequence-level normalization, or a non-neutral readout can break zero
extension. The current OTransformer result covers one feature-token host and
does not include OStandardize. The separate TabPFN v3 retrofit evaluates the
hidden-carrier realization on one pretrained host without fine-tuning; it does
not establish arbitrary-host safety or universal task-level no-loss.

\paragraph{Missingness.}
Observed zero, missing, NaN, learned mask states, and query placeholders have
different semantics. Mapping missing data to an exact-zero carrier is a
promising option only when missingness is declared uninformative. Informative
missingness should retain a nonzero side channel. Establishing when either
choice is appropriate is outside the present experiments.

\section{Conclusion}
\label{sec:conclusion}

Attention masks answer a relation-level question: which query--source
interactions are permitted. Token presence answers the complementary
state-level question of whether a representation participates at all. We
assign every hidden carrier one coefficient
\(p_i=\presence_\tau(h_i)\). Its receiver use gates newly emitted
information; its source use controls mass in context support. OAttention is
the nonlocal attention realization of these two placements, retaining the
standard score, visibility relation, exponential competition, and value
aggregation while making the origin a zero element.

This single coefficient also organizes the rest of token dynamics. Local
updates yield OFFN, ONorm, and OInject; support-weighted moments yield
OStandardize; and the O-Closure law propagates null insertion through residual
updates and composition. The components are therefore consequences of one
token-participation variable rather than a collection of analogous gates.

The experiments address three bounded questions. Operator checks reproduce the
canonical equation and null contracts at finite precision; a cloned pretrained
TabPFN v3 host exhibits a calibrated near-identity active path without
fine-tuning; and a two-block negative control shows why OAttention alone is
insufficient when ordinary host components reactivate NULL. These results do
not establish arbitrary-host safety, universal task-level non-regression,
learned zero-attractor dynamics, or a general missing-value semantics. Whether
the construction preserves weak active signals across model families, induces
useful zero-attractor dynamics, or provides an appropriate semantics for
missing values remains open.

\paragraph{Reproducibility.}
The accompanying artifact contains the source package, tests, experiment
runners, configurations, and machine-readable results. Exact reproduction
commands, environment information, and CPU/CUDA provenance are reported in
\Cref{app:artifacts}.

\bibliographystyle{plainnat}
\bibliography{references}

\clearpage
\appendix

\section{Additional mathematical details}
\label{app:mathematical-details}

\subsection{Projection-boundary design variation}
\label{app:qkv-reference-form}

Presence may be evaluated after projection rather than on the shared hidden
carrier. For head $a$, define
\begin{equation}
  p_{ia}^{Q}=\presence_\tau(q_{ia}),\qquad
  p_{ja}^{K}=\presence_\tau(k_{ja}),
  \qquad
  \OA_{ia}^{QK}
  =p_{ia}^{Q}
   \frac{\sum_jm_{ij}p_{ja}^{K}e^{s_{ija}}v_{ja}}
        {\varepsilon_{\mathrm{den}}+
         \sum_jm_{ij}p_{ja}^{K}e^{s_{ija}}}.
  \label{eq:qkv-reference}
\end{equation}
This form satisfies null-query and null-source contracts at the projected
Q/K/V boundary. It differs from canonical hidden-carrier OAttention whenever
a nonzero $h_i$ projects to $q_{ia}=0$ or $k_{ia}=0$ in one head: the
projection-boundary form suppresses that role in that head, whereas the
canonical form retains the shared token presence. Under zero-preserving
projections, both forms preserve a hidden zero. They are therefore related
designs with different declared carrier boundaries, not notation variants.

\subsection{Conditional-kernel interpretation}

For a token transition kernel
\(K_I(dH'\mid H,c)\), the deterministic zero-extension law has a direct
probabilistic analogue. After adjoining a null index \(j\),
\begin{equation}
  K_{I\cup\{j\}}
  \bigl(dH'_I,dh'_j\mid H_I\oplus_j0,c^+\bigr)
  =K_I(dH'_I\mid H_I,c)\,\delta_0(dh'_j).
  \label{eq:kernel-closure}
\end{equation}
The first factor states that a null context token does not change the
conditional dynamics of old tokens. The second states that the null position
is absorbing. This factorization is stronger than saying only that the new
position has zero output.

A query placeholder has a different factorization. A full structural NULL has
one hidden-carrier presence (p=0). A read-only query hole instead keeps a
nonzero carrier (p>0) so its next state can depend on context, while the
caller masks its source column or otherwise excludes it from source support.
In the Q/K/V reference this distinction can also be expressed as
\((p^Q>0,p^K=0)\). A learned \texttt{[MASK]} token can be nonzero and
information-bearing. These states should not be identified merely because
they occupy the same table cell.

\subsection{Compatibility of metadata}

Equation~\eqref{eq:o-closure} compares an original indexed family with its
zero extension. The comparison requires metadata on old indices to remain
fixed. Appending a token to a table with fixed feature identifiers is
compatible. Inserting before a sequence position and then recomputing absolute
position indices may not be compatible because old carriers have changed.
OInject prevents an
exact-zero token from receiving additive metadata; it
does not make arbitrary re-indexing invariant.

\subsection{Why output gating is insufficient for coupled maps}

Let \(A_i(H)\) be a cross-token aggregate and consider
\(\widetilde A_i(H)=p_iA_i(H)\). If \(p_i=0\), the output at position \(i\)
vanishes. For an old active position \(r\), however,
\(\widetilde A_r(H\oplus0)=p_rA_r(H\oplus0)\) can differ from
\(p_rA_r(H)\). Ordinary softmax attention and unweighted standardization are
examples: the inserted zero changes a denominator or empirical moment. Source
presence must therefore enter the coupled support itself, as in
Eqs.~\eqref{eq:oattention-measure} and \eqref{eq:ostandardize}.

\section{Experimental protocols and reproducibility}
\label{app:experimental-protocols}

\subsection{Operator implementation details}

The package exposes two explicit classes. \texttt{HiddenCarrierOAttention}
computes one token presence from each hidden carrier and broadcasts it across
heads. \texttt{OAttention} retains the Q/K/V-boundary design variation,
computing query and key presence per head. Both paths compute norms, scores,
exponentials, and normalizers in fp32, then cast outputs and reported weights
to the caller dtype. They branch on a zero-length key axis and map fully masked
rows to zero. Boolean masks use \texttt{True} for visibility; additive masks
map excluded edges to \(-\infty\) before exponentiation. Under grouped-query
attention, the canonical source coefficient is broadcast across expanded
query heads, while the reference key coefficient follows the KV-head repeat.

\subsection{Datasets, splits, and metrics}

All real-data experiments use datasets distributed with scikit-learn. We
split rows into 60\% training, 20\% validation, and 20\% test subsets;
classification splits are stratified. Feature standardization is fitted on
the training split and applied unchanged to validation and test. For Diabetes,
the target transform is also fitted on training targets, and RMSE and MAE are
reported after inversion to original units.

\begin{table}[H]
  \centering
  \small
  \caption{Real datasets used in the learned-model evaluations.}
  \label{tab:datasets}
  \begin{tabular}{@{}lrrrl@{}}
    \toprule
    Dataset & Rows & Features & Classes & Task \\
    \midrule
    Iris & 150 & 4 & 3 & classification \\
    Wine & 178 & 13 & 3 & classification \\
    Breast Cancer Wisconsin & 569 & 30 & 2 & classification \\
    Digits & 1,797 & 64 & 10 & classification \\
    Diabetes & 442 & 10 & -- & regression \\
    \bottomrule
  \end{tabular}
\end{table}

Classification metrics are accuracy, balanced accuracy, and cross-entropy.
Regression metrics are RMSE and MAE. Unless stated otherwise, learned-model
comparisons use seeds \(11,23,37\). The learned-model experiments are small,
controlled architecture studies rather than broad benchmarks.

\subsection{Canonical hidden-carrier model configurations}

The minimal scalar-feature adapter maps a standardized scalar \(z_j\) to
\(h_j=z_jw_j\in\R^{64}\) without token bias. It applies one four-head
self-attention layer, a residual connection, mean pooling over feature tokens,
and a task head. The standard and hidden-carrier arms share initialization,
optimizer, split, and training budget. OAttention uses \(\tau=10^{-6}\). The
projection-boundary experiments use the same host protocols but evaluate a
different declared carrier boundary.

The test-time insertion experiment uses a two-block, 32-dimensional,
four-head FT-style host with a learned CLS readout. Exact zero tokens are
appended after the ordinary feature and CLS carriers. The insertion comparison
is made within the trained model between the original and extended sequence.
The reported prediction shift is therefore not a change in task loss.

The training-time zero-column study uses the same two-block hidden size and
head count, a bias-free feature tokenizer, affine-free LayerNorm, bias-free
FFN, and learned CLS readout. Exact zero columns are added after train-fitted
preprocessing to training, validation, and test sets. Standard and
hidden-carrier OAttention arms use the same initial state for each task, seed,
and zero count.

\subsection{Pretrained TabPFN v3 retrofit protocol}

The hidden-carrier host comparison uses an independent copy of the pretrained
TabPFN v3 regressor checkpoint (vendor package 8.0.1) on an NVIDIA GB10 GPU.
The baseline and both O arms are separate model instances loaded from the same
checkpoint; no parameter is fine-tuned. The OAttention arm replaces attention
participation at the feature-distribution, column-aggregation, and ICL
boundaries. Full-O additionally applies ONorm, OFFN, and hidden-carrier gates to
the two target-encoding additions in the v3 forward path. The upstream
preprocessing standardizer remains unchanged, so this is not an
OStandardize task experiment.

The cached OpenML data IDs are 560 (bodyfat), 44959
(concrete\_compressive\_strength), 505 (tecator), 507 (space\_ga), 227
(cpu\_small), and 189 (kin8nm). We use seeds $11,23,37$, one estimator per
arm, and deterministic caps of 512 training and 256 test rows for larger
datasets. RMSE is measured in original target units and $R^2$ is computed on
each held-out test split. The main table uses $\tau=10^{-8}$ as a
near-identity active-path control. A separate $\tau=1$ stress matrix is
reported below; it is not pooled into the main table.

\subsection{OStandardize protocol}

The machine-readable OStandardize sweep fixes \(\tau=10^{-6}\),
\(\varepsilon_{\mathrm{var}}=10^{-6}\), and token axis 1. It evaluates fp32
shapes \((2,5,4)\), \((1,7,8)\), and \((3,4,16)\), together with bf16 and
fp16 cases on \((2,5,4)\). Each case checks:
\begin{itemize}[leftmargin=1.5em,itemsep=0.15em]
  \item preservation of old support, mean, variance, and outputs under one
        zero inserted at the tested boundary and interior positions;
  \item exact-zero output for the inserted positions;
  \item permutation equivariance and finite gradients;
  \item all-null, zero-length, and singleton active support.
\end{itemize}
Constant active support is covered separately by the component unit tests.
The singleton case also verifies the selected layerwise semantics: centering
produces an exact-zero state, and a subsequent biased OFFN branch emits zero
when presence is recomputed from that state.

\subsection{OTransformer protocol}

The OTransformer experiment uses two blocks, hidden dimension 32, four heads,
FFN dimension 128, 100 epochs, learning rate \(10^{-3}\), weight decay
\(10^{-4}\), and \(\tau=\varepsilon_{\mathrm{den}}=10^{-6}\). Exact zero
tokens are inserted after feature tokenization and before position encoding
and Transformer blocks. The factorial comprises five datasets, three seeds,
three training zero counts \(k\in\{0,1,4\}\), and four arms, for 180 fits.

\begin{table}[H]
  \centering
  \small
  \caption{Arms in the scoped OTransformer experiment. The first three share
  exact initialization; the structural control is initialized separately.}
  \label{tab:otransformer-arms}
  \begin{tabularx}{\linewidth}{@{}lXXX@{}}
    \toprule
    Arm & Attention & Encoding and normalization & FFN \\
    \midrule
    Standard & softmax & additive encoding; affine LayerNorm & biased \\
    OA-only & hidden-carrier OAttention & additive encoding; affine LayerNorm & biased \\
    OTransformer & hidden-carrier OAttention & OInject; ONorm-wrapped affine LayerNorm & OFFN-wrapped biased FFN \\
    Structural control & hidden-carrier OAttention & no additive encoding; affine-free LayerNorm & bias-free, O-gated \\
    \bottomrule
  \end{tabularx}
\end{table}

The primary diagnostics are per-layer inserted-state norm, old-token and CLS
state shift, prediction shift, norm and presence quantiles, exact/near-zero
fractions, and finiteness. The task head is outside the token-dynamics theorem;
a bias in the head can produce a nonzero all-null prediction even when every
token state remains zero.

\section{Additional empirical results}

\subsection{Q/K/V-boundary reference operator properties}
\label{app:qkv-reference-results}

A separate projection-boundary study covers five random seeds and three
\((H,Q,K,d)\) configurations, for 15 cases at
\(\tau=\varepsilon_{\mathrm{den}}=10^{-6}\). It evaluates exact and fully
masked boundaries, insertion at the Q/K/V boundary, the active-state
comparison with softmax, near-origin scaling, grouped-query attention, masks,
bf16, and gradients.

\begin{table}[H]
  \centering
  \small
  \caption{Numerical verification of the Q/K/V-boundary design variation. The
  nonzero insertion errors are floating-point residuals on pre-existing
  outputs or weights.}
  \label{tab:operator-results}
  \begin{tabularx}{\linewidth}{@{}lX@{}}
    \toprule
    Property & Result \\
    \midrule
    Finiteness & 15/15 seed--shape cases finite \\
    Exact zero boundaries & Null query, empty support, and fully masked output and weights: \(0\) \\
    Null-source insertion & Old-output \(L_\infty\le2.4\times10^{-7}\); old-weight \(L_\infty\le6.0\times10^{-8}\) \\
    Self-attention insertion & Old-output \(L_\infty\le4.8\times10^{-7}\); inserted-query output \(0\) \\
    Near-origin behavior & Log--log output/query slope \(1.9997\) over scales \(10^{-7}\) to \(10^{-5}\) \\
    Active path & Relative output \(L_\infty\le3.7\times10^{-7}\) at the evaluated default scale \\
    Interfaces & GQA, boolean/additive masks, bf16, and tested gradients finite \\
    \bottomrule
  \end{tabularx}
\end{table}

The measured near-origin slope matches the quadratic expansion of
Eq.~\eqref{eq:presence}. The active-path difference is a finite-scale numerical
observation, while Eq.~\eqref{eq:vanilla-limit} supplies the asymptotic
statement. Because the presence boundary differs, these measurements do not
constitute evidence for the canonical hidden-carrier operator.

\subsection{Finite-gate stress test for the pretrained TabPFN v3 retrofit}
\label{app:tabpfn-v3-stress}

The main text uses $\tau=10^{-8}$ to isolate the minimal-change regime. To
make the calibration trade-off visible, we repeat the same 54 inference fits
with $\tau=1$. OAttention remains close to the baseline, but applying the
same gate to every local branch can materially alter a pretrained active path.
The entries below are means over seeds; positive RMSE deltas and negative
$R^2$ deltas indicate degradation.

\begin{table}[H]
  \centering
  \scriptsize
  \caption{Finite-$\tau$ stress matrix on the cloned pretrained TabPFN v3
  host. Deltas are relative to the unmodified checkpoint.}
  \label{tab:tabpfn-v3-stress}
  \begin{tabular}{@{}lrrrr@{}}
    \toprule
    Dataset & OA $\Delta$RMSE & Full-O $\Delta$RMSE & OA $\Delta R^2$ & Full-O $\Delta R^2$ \\
    \midrule
    bodyfat & +0.532\% & +95.354\% & +0.000028 & +0.005431 \\
    concrete & +0.162\% & +24.442\% & -0.000178 & -0.034037 \\
    cpu\_small & -0.054\% & +5.238\% & +0.000031 & -0.002766 \\
    kin8nm & -0.112\% & +39.797\% & +0.000277 & -0.117006 \\
    space\_ga & -0.126\% & +17.324\% & +0.000649 & -0.096142 \\
    tecator & +0.429\% & +100.143\% & -0.000008 & -0.002602 \\
    \midrule
    Mean over 18 cells & +0.138\% & +47.050\% & +0.000133 & -0.041187 \\
    \bottomrule
  \end{tabular}
\end{table}

The large Full-O changes are not a refutation of the zero-extension contract:
they are an active-path calibration result for a pretrained model that was not
trained with repeated gates. The near-identity matrix and this stress matrix
should therefore be read together.

\subsection{Canonical hidden-carrier adapter calibration}
\label{app:adapter-table}

The broad adapter receipt covers seven tasks and three matched seeds. The table
reports test means; positive regression deltas are worse because RMSE and MAE
are lower-is-better metrics. These numbers are included to expose the scale of
the active-path perturbation, not to rank the two attention mechanisms.

\begin{table}[H]
  \centering
  \scriptsize
  \caption{Canonical hidden-carrier adapter rerun. Deltas are hidden-carrier
  OAttention minus standard softmax.}
  \label{tab:adapter-calibration}
  \begin{tabular}{@{}lcrrr@{}}
    \toprule
    Task & Metric & Standard & Hidden-carrier OA & $\Delta$ \\
    \midrule
    Iris & accuracy & .9111 & .9111 & +.0000 \\
    Wine & accuracy & .9630 & .9630 & +.0000 \\
    Breast Cancer & accuracy & .9708 & .9708 & +.0000 \\
    Digits & accuracy & .8972 & .8954 & -.0019 \\
    Synthetic classification & accuracy & .6889 & .6900 & +.0011 \\
    Diabetes & RMSE & 54.6873 & 54.6782 & -.0092 \\
    Friedman1 & RMSE & 2.5011 & 2.5016 & +.0005 \\
    \bottomrule
  \end{tabular}
\end{table}

The corresponding cross-entropy and MAE values, per-seed outcomes, presence
quantiles, and runtime metadata remain in the machine-readable receipt.

\subsection{Test-time NULL insertion}
\label{app:insertion-check}

We append one or four exact zero-vector tokens after token construction in a
two-block, 32-dimensional host and compare predictions with the unextended
sequence. The inserted state is measured after the final block. Across all 15
task--seed pairs, every inserted hidden-carrier OAttention token remained
exactly zero in the recorded post-block tensor. Standard attention activated
the new positions and changed predictions more strongly
(\Cref{tab:insertion-results}).

\begin{table}[H]
  \centering
  \scriptsize
  \caption{Test-time insertion over five datasets and three seeds. Prediction
  shift is RMS over test predictions; state is the maximum absolute inserted
  coordinate.}
  \label{tab:insertion-results}
  \begin{tabular}{@{}lrrrr@{}}
    \toprule
    Inserted & \multicolumn{2}{c}{Prediction RMS shift} &
    \multicolumn{2}{c}{Inserted-state \(L_\infty\)} \\
    tokens & Standard & Hidden-carrier OA & Standard & Hidden-carrier OA \\
    \midrule
    1 & \(7.39\times10^{-2}\) & \(3.69\times10^{-5}\) & 1.83--7.42 & 0 \\
    4 & \(2.18\times10^{-1}\) & \(1.64\times10^{-4}\) & 1.84--7.42 & 0 \\
    \bottomrule
  \end{tabular}
\end{table}

\subsection{Training-time exact-zero columns}
\label{app:zero-column-table}

\begin{table}[H]
  \centering
  \scriptsize
  \caption{Mean test metric over three seeds in the canonical hidden-carrier
  training-time zero-column study, reported as standard/hidden-carrier
  OAttention. For Diabetes, lower RMSE is better.}
  \begin{tabular}{@{}llrrrr@{}}
    \toprule
    Task & Metric & \(k=0\) & \(k=1\) & \(k=2\) & \(k=4\) \\
    \midrule
    Iris & accuracy & .8667/.9222 & .8667/.9111 & .8778/.9111 & .9333/.8889 \\
    Wine & accuracy & .9444/.9444 & .9444/.9444 & .9444/.9444 & .9444/.9444 \\
    Breast Cancer & accuracy & .9678/.9678 & .9678/.9678 & .9678/.9678 & .9678/.9678 \\
    Digits & accuracy & .9204/.9213 & .9194/.9213 & .9194/.9213 & .9194/.9213 \\
    Diabetes & RMSE & 62.6683/62.6689 & 62.7226/62.6680 & 62.7453/62.6682 & 62.8580/62.6688 \\
    \bottomrule
  \end{tabular}
\end{table}

These values provide a bounded matched comparison only. The study was not
designed with a formal non-inferiority margin, and differences across tasks
and seeds should not be aggregated into a universal no-loss statement.

\subsection{Component-level checks}

The component tests use biased FFNs, affine normalizers, and arbitrary
additive encodings to verify exact-zero output or update at the declared
boundary. They also check convergence toward the wrapped active path as
\(\tau\to0^+\), finite gradients, tested dtypes, shape validation, block-level
insertion, the OA-only failure boundary, and one-component-at-a-time
ablations. These tests establish implementation agreement with the displayed
definitions; they do not show that every possible wrapped module is safe.

\section{Reproducibility artifacts}
\label{app:artifacts}

The implementation is in the \path{zero_activity_attention} Python package.
The principal experiment entry points are
\path{run_systematic_operator_sweep.py},
\path{run_hidden_carrier_oattention_sweep.py},
\path{run_ft_transformer_paper_suite.py},
\path{run_zero_column_training_factorial.py},
\path{run_o_standardize_sweep.py}, and
\path{run_otransformer_suite.py}.

The following machine-readable artifacts support the numerical values in the
publication manuscript. Paths are relative to
\path{experiments/results/}:
\begin{description}[leftmargin=3.1cm,style=nextline,itemsep=0.25em]
  \item[Hidden-carrier OAttention]
    \path{remote/hidden-carrier-oattention-dgx2-20260821/hidden-carrier-oattention-dgx2-20260821-v1.json}
  \item[Operator sweep]
    \path{systematic-operator-sweep-dgx2.json}
  \item[Matched adapters]
    \path{adapter-nonregression-dgx2.json},
    \path{digits-adapter-nonregression-dgx2.json}, and
    \path{regression-adapter-nonregression-dgx2-fixed.json}
  \item[Canonical matched-task rerun]
    \path{hidden-carrier-task-rerun-20260821-v1/broad-adapter-hidden-carrier-20260821.json},
    \path{hidden-carrier-task-rerun-20260821-v1/ft-transformer-hidden-carrier-20260821.json}, and
    \path{hidden-carrier-task-rerun-20260821-v1/zero-column-hidden-carrier-20260821.json}
  \item[Test-time insertion]
    \path{ft-transformer-paper-suite-dgx2-20260820.json}
  \item[Training-time zeros]
    \path{remote/zero-column-training-factorial-dgx2-20260820/zero-column-training-factorial-dgx2-20260820-v2.json}
  \item[OStandardize]
    \path{remote/ostandardize-20260821/ostandardize-sweep-20260821-v2.json}
  \item[OTransformer]
    \path{remote/otransformer-dgx2-20260821/otransformer-suite-dgx2-20260821-v4.json}
  \item[Canonical OTransformer rerun]
    \path{hidden-carrier-task-rerun-20260821-v1/otransformer-hidden-carrier-20260821.json}
  \item[Pretrained TabPFN v3 near-identity]
    \path{tabpfn3-o-closure-20260821-v1/openml_regression_matrix_tau1e-8.json}
  \item[Pretrained TabPFN v3 finite-$\tau$ stress]
    \path{tabpfn3-o-closure-20260821-v1/openml_regression_matrix_tau1_final.json}
  \item[Protocol clarification]
    \path{remote/otransformer-dgx2-20260821/protocol-clarification-v4.json}
\end{description}

The canonical task rerun uses source SHA-256 beginning
\texttt{f9fc6e9b52d968a2}; the complete digest is stored in each receipt.
Its receipt digests begin with \texttt{c8198791}, \texttt{3dbd656c},
\texttt{d9484ac3}, and \texttt{24e4afe0} for the adapter, FT, zero-column,
and OTransformer matrices, respectively. The historical hidden-carrier
operator source and runner SHA-256 values begin with
\texttt{a87ac5c8} and \texttt{71a1c463}; the OStandardize source and runner
values begin with \texttt{e4276970} and \texttt{50af6245}; the historical
OTransformer source and runner values begin with \texttt{e4276970} and
\texttt{8e3fc992}. The complete
digests for the TabPFN v3 runner and checkpoint begin with
\texttt{e73c1134} and \texttt{311ce18d}; its copied vendor architecture begins
with \texttt{733f0f24}. Environment metadata, per-run configurations, and
task-level outputs
are stored in the artifacts above. Historical runs and release verification
records remain in the repository but are not pooled into the publication
tables.

\end{document}